\documentclass{article} 
\usepackage{nips12submit_e,times}

\usepackage{graphicx} 
\usepackage{subfigure}
\usepackage{algorithm}
\usepackage{algorithmic}
\usepackage{amssymb,amsthm,amsmath}

\newcommand{\beq}{\vspace{0mm}\begin{equation}}
\newcommand{\eeq}{\vspace{0mm}\end{equation}}
\newcommand{\beqs}{\vspace{0mm}\begin{eqnarray}}
\newcommand{\eeqs}{\vspace{0mm}\end{eqnarray}}
\newcommand{\barr}{\begin{array}}
\newcommand{\earr}{\end{array}}

\newcommand{\rv}{\boldsymbol{r}}

\newcommand{\xv}{\boldsymbol{x}}

\newcommand{\lambdav}[0]{{\boldsymbol{\lambda}}}

\newcommand{\E}{\mathbb{E}}

\newtheorem{thm}{Theorem}[section]

\newtheorem{lem}[thm]{Lemma}

\title{Augment-and-Conquer Negative Binomial Processes}

\author{
Mingyuan Zhou\\
Dept. of Electrical and Computer Engineering\\
Duke University, Durham, NC 27708 \\
\texttt{mz1@ee.duke.edu} \\
\And
Lawrence Carin \\
Dept. of Electrical and Computer Engineering\\
Duke University, Durham, NC 27708 \\
\texttt{lcarin@ee.duke.edu} \\
}

\nipsfinalcopy 

\begin{document}

\maketitle

\begin{abstract}
By developing data augmentation methods unique to the negative binomial (NB) distribution, we unite seemingly disjoint count and mixture models  under the  NB process framework. We develop fundamental properties of the models and derive efficient Gibbs sampling inference. We show that the  gamma-NB process can be reduced to the hierarchical Dirichlet process with normalization, highlighting its unique theoretical, structural and computational advantages.
A variety of NB processes with distinct sharing mechanisms are constructed and applied to topic modeling, with connections to existing algorithms, showing the importance of inferring both the NB dispersion and probability parameters.
\end{abstract}

\section{Introduction}

There has been increasing interest in count modeling using the Poisson process, geometric process 
\cite{PoissonP,InfGaP, Thibaux,Miller} and recently the negative binomial (NB) process  \cite{BNBP_PFA,NBPJordan}. Notably, it has been independently shown in \cite{BNBP_PFA} and \cite{NBPJordan}  that the NB process, originally constructed for count analysis, can be naturally applied for mixture modeling of \emph{grouped} data $\xv_1,\cdots,\xv_J$, where each group $\xv_j=\{x_{ji}\}_{i=1,N_j}$.
For a territory long occupied by the hierarchical Dirichlet process (HDP) \cite{HDP} and related models,  the inference of which may require substantial bookkeeping and suffer from slow convergence \cite{HDP}, the discovery of the NB process for mixture modeling can be significant. As the seemingly distinct problems of count and mixture modeling are united under the NB process framework, new opportunities emerge for better data fitting, more efficient inference and more flexible model constructions. However, neither \cite{BNBP_PFA} nor  \cite{NBPJordan} explore the properties of
the NB distribution deep enough to achieve fully tractable closed-form inference.
Of particular concern is the NB dispersion parameter, which was simply fixed or empirically set~\cite{NBPJordan}, or inferred with a Metropolis-Hastings algorithm~\cite{BNBP_PFA}. Under these limitations, both papers fail to  reveal the connections of the NB process to the
HDP,
and thus may lead to false assessments on comparing their modeling abilities.

We perform joint count and mixture modeling under the NB process framework, using completely random measures \cite{PoissonP,JordanCRM,Wolp:Clyd:Tu:2011} that are simple to construct and amenable for posterior computation.
We propose to augment-and-conquer the NB process:  by ``augmenting'' a NB process into both the gamma-Poisson and compound Poisson representations, we ``conquer'' the unification of count and mixture modeling, the analysis of fundamental model properties, and the derivation of efficient  Gibbs sampling inference.
We make two additional contributions:  1) we construct a gamma-NB process, analyze its properties and show how its normalization leads to the HDP, highlighting its unique theoretical, structural and computational advantages  relative to the HDP.
2) We show that a variety of NB processes can be constructed with distinct model properties, for which the shared random measure can be selected from completely random measures such as the gamma, beta, and beta-Bernoulli processes; we compare their performance on topic modeling, a typical example for mixture modeling of grouped data, and show the importance of inferring both the NB dispersion and probability parameters, which respectively govern the overdispersion level and the variance-to-mean ratio in count modeling.

\subsection{Poisson process for count and mixture modeling}\label{sec:PP}
\vspace{-2.5mm}
Before introducing the NB process, we first illustrate how the seemingly distinct problems of count and mixture modeling can be united under the Poisson process.
Denote $\Omega$ as a measure space and for each Borel set $A\subset \Omega$, denote $X_j(A)$ as a count random variable describing the number of observations in $\xv_j$ that reside within $A$. Given grouped data $\xv_1,\cdots,\xv_J$, for any measurable disjoint partition  $A_1,\cdots,A_Q$ of $\Omega$, we aim to jointly model the count random variables $\{X_j(A_q)\}$.
A natural choice would be to define a Poisson process $X_j\sim\mbox{PP}(G)$, with a shared completely random measure $G$ on $\Omega$, such that
$
X_j(A)\sim\mbox{Pois}\big(G(A)\big)
$ for each $A\subset\Omega$. Denote $G(\Omega) = \sum_{q=1}^Q G(A_q)$ and $\widetilde{G} = G/G(\Omega)$. 
 Following Lemma 4.1 of \cite{BNBP_PFA}, the joint distributions of $X_j(\Omega),X_j(A_1),\cdots,X_j(A_Q)$ are equivalent under the following two expressions:
\vspace{-1.2mm}\beqs\small  \label{eq:Pois}& X_j(\Omega) = \sum_{q=1}^Q X_j(A_q),~~~~~~X_j(A_q)\sim\mbox{Pois}\big(G(A_q)\big);\\
 \label{eq:Mult}\small
&X_j(\Omega) \sim \mbox{Poisson}(G(\Omega)),~~~[X_j(A_1),\cdots,X_j(A_q)]\sim \mbox{Mult}\big(X_j(\Omega); \widetilde{G}(A_1),\cdots,\widetilde{G}(A_Q)\big).\vspace{-2.8mm}
\eeqs
 Thus the Poisson process provides not only  a way to generate independent counts from each $A_q$, but also a mechanism for mixture modeling, which allocates the observations into any measurable disjoint partition $\{A_q\}_{1,Q}$ of $\Omega$, conditioning on
 $X_j(\Omega)$ and the normalized mean measure $\widetilde{G}$.

 To complete the model, we may place a gamma process \cite{Wolp:Clyd:Tu:2011} prior on the shared  measure as $G\hspace{-1.2mm}\sim\hspace{-1.2mm}\mbox{GaP}(c,G_0)$, with concentration parameter $c$ and base measure $G_0$,
 such that $G(A)\hspace{-1.2mm}\sim\hspace{-1.2mm}\mbox{Gamma}(G_0(A),1/c)$ 
 for each $A\hspace{-1.0mm}\subset\hspace{-1.0mm} \Omega$, where $G_0$ can be continuous, discrete or a combination of both. Note that $\widetilde{G} = G/G(\Omega)$ now becomes a Dirichlet process (DP) as $\widetilde{G}\hspace{-0.8mm} \sim \hspace{-0.8mm}\mbox{DP}(\gamma_0,\widetilde{G}_0)$, where $\gamma_0\hspace{-0.8mm}=\hspace{-0.8mm}G_0(\Omega)$ and $\widetilde{G}_0 \hspace{-0.8mm}=\hspace{-0.8mm} G_0/\gamma_0$. The normalized gamma
 representation of the DP is discussed in \cite{ferguson73,ishwaran02,Wolp:Clyd:Tu:2011} and has been used to construct the group-level DPs for an HDP 
  \cite{DILN}.
The Poisson process\linebreak has an equal-dispersion assumption for count modeling. 
 As shown in (\ref{eq:Mult}), the construction of Poisson processes with a shared gamma process mean measure implies the same mixture proportions across groups, which is essentially the same as the DP when used for mixture modeling when the total counts $\{X_j(\Omega)\}_j$ are not treated as random variables. 
 This motivates us to consider adding an additional layer
 or using a different distribution other than the Poisson to model the counts. As shown below, the NB distribution is an ideal candidate, not only because it allows overdispersion, 
 but also because it can be augmented into both a gamma-Poisson and a compound Poisson representations.

\vspace{-3mm}
\section{Augment-and-Conquer the Negative Binomial Distribution}
\vspace{-2.5mm}

The NB distribution $m\sim\mbox{NB}(r,p)$ has the probability mass function (PMF) $f_M(m) = \frac{\Gamma(r+m)}{m!\Gamma(r)}(1-p)^r p^m$. 
It  has a mean $\mu = {rp}/(1-p)$ smaller than the variance $\sigma^2 = {rp}/{(1-p)^2} = \mu + r^{-1}\mu^2$, with the variance-to-mean ratio (VMR) as $(1-p)^{-1}$ and the overdispersion level (ODL, the coefficient of the quadratic term in $\sigma^2$) as $r^{-1}$. It has been widely investigated and applied to numerous scientific studies  \cite{NB_Fitting_53,Cameron1998,WinkelmannCount}.
  The NB distribution can be augmented into a gamma-Poisson construction 
  as
$m\sim \mbox{Pois}(\lambda),~\lambda\sim \mbox{Gamma}\left(r,{p}/{(1-p)}\right)$, where the gamma distribution is parameterized by its shape $r$ and scale $p/(1-p)$.
It can also be augmented under a compound Poisson representation \cite{LogPoisNB} as
$ 
m = \sum_{{t}=1}^{l}u_{{t}},~ u_{{t}}\sim\mbox{Log}(p), ~ l \sim \mbox{Pois}(-r\ln(1-p))
$, 
where $u\sim \mbox{Log}(p)$ is the logarithmic distribution~\cite{johnson2005univariate} with 
probability-generating function (PGF)
$ C_U(z) = {\ln(1-pz)}/{\ln(1-p)},~~|z|<{p^{-1}}.
$ In a slight abuse of notation, but for added conciseness, in the following discussion
we use $m\sim \sum_{t=1}^l \mbox{Log}(p)$ to denote $m = \sum_{{t}=1}^{l}u_{{t}},~ u_{{t}}\sim\mbox{Log}(p)$.

 The inference of the NB dispersion parameter $r$ 
 has long been a challenge 
 \cite{NB_Fitting_53,NB_MQL,
 NB_Bio_2008}. In this paper, we first place a gamma prior on  it as $r\sim\mbox{Gamma}(r_1,1/c_1)$. We then 
 use Lemma \ref{lem:divide} (below) to infer a latent count $l$ for each $m\sim\mbox{NB}(r,p)$ conditioning on $m$ and $r$. Since $l\sim \mbox{Pois}(-r\ln(1-p))$ by construction, we can use the gamma Poisson conjugacy to update $r$. Using Lemma \ref{lem:merge} (below), we can further infer an augmented latent count  $l'$ for each $l$, 
 and then use these latent counts 
 to update $r_1$, assuming $r_1\sim\mbox{Gamma}(r_2,1/c_2)$. Using  Lemmas \ref{lem:divide} and  \ref{lem:merge}, we can continue this process repeatedly,
 suggesting that we may build a NB process to
 model data that have subgroups within groups. The conditional posterior of the latent count $l$
 was first derived by us  but was not given an analytical form \cite{LGNB}. Below we explicitly derive the PMF of $l$, shown in (\ref{eq:CRT}), and find that it exactly represents the distribution of the random number of  tables occupied by $m$ customers  in a Chinese restaurant process with concentration parameter $r$ \cite{DP_Mixture_Antoniak,Escobar1995,HDP}. We denote $l\sim\mbox{CRT}(m,r)$ as a Chinese restaurant  table (CRT) count random variable with such a PMF and as proved  in the supplementary material, we can sample it as $l=\sum_{n=1}^m b_n,~b_n\sim\mbox{Bernoulli}\left({r}/(n-1+r)\right)$.

 Both the gamma-Poisson and compound-Poisson augmentations of the NB distribution and Lemmas \ref{lem:divide} and \ref{lem:merge} are key ingredients of this paper. We will show that these augment-and-concur methods not only unite  count and mixture modeling and provide efficient inference, but also, as shown in Section~\ref{sec:GNBP}, let us examine the posteriors to understand fundamental properties of the NB processes,  clearly revealing connections to previous nonparametric Bayesian mixture models.

\begin{lem}\label{lem:divide} Denote $s(m,j)$ as Stirling numbers of the first kind \cite{johnson2005univariate}.
Augment $m\sim\emph{\mbox{NB}}(r,p)$ under the compound Poisson representation as $m \sim\sum_{{t}=1}^{l} \emph{\mbox{Log}}(p),~ l\sim\emph{\mbox{Pois}}(-r\ln(1-p))$, then the conditional posterior of $l$ has PMF
\beqs\label{eq:CRT}
&\emph{\mbox{Pr}}(l = j| m,r) =  {\frac{\Gamma(r)}{\Gamma(m+r)}}|s(m,j)| r^j, ~~ j=0,1,\cdots,m.
\eeqs
\end{lem}
\vspace{-4.0mm}
\begin{proof}
Denote $w_{j} \sim \sum_{{t}=1}^j \mbox{Log}(p)$, $j=1,\cdots,m$. Since $w_{j}$ is the summation of $j$ iid  $\mbox{Log}(p)$ random variables,
the PGF of $w_{j}$ becomes
$
C_{W_{j}}(z)=C_{U}^j(z)=
\left[{\ln(1-pz)}/{\ln(1-p)}\right]^j,~ |z|<{p^{-1}}
$. Using   the property that $[\ln(1+x)]^j = j!\sum_{n=j}^\infty\frac{s(n,j)x^n}{n!}$ \cite{johnson2005univariate}, we have $ 
\mbox{Pr}(w_{j}=m) = {C_{W_{j}}^{(m)}(0)}/{m!} = (-1)^m p^j j! s(m,j)/(m![\ln(1-p)]^j). $ 
Thus  for $0\le j \le m$, we have
$ 
\mbox{Pr}(L = j|m,r) \propto  \mbox{Pr}(w_{j}=m) \mbox{Pois}(j;-r\ln(1-p)) \propto |s(m,j)| r^j.
$ 
Denote $S_r(m) = \sum_{j=0}^m |s(m,j)| r^j$, we have $S_r(m) = (m-1+r)S_r(m-1)=\cdots=\prod_{n=1}^{m-1} (r+n) S_r(1)  = \prod_{n=0}^{m-1} (r+n)=\frac{\Gamma(m+r)}{\Gamma(r)}$.
\end{proof}
\vspace{-3mm}
\begin{lem}\label{lem:merge}  Let $m\sim\emph{\mbox{NB}}(r,p),~r\sim\emph{\mbox{Gamma}}(r_1,1/c_1)$, denote $p' = \frac{-\ln(1-p)}{c_1-\ln(1-p)}$, then $m$ can also be generated from a compound distribution as 
\vspace{-0.8mm}\beqs\label{eq:merge}
&m\sim \sum_{{t}=1}^{l} \emph{\mbox{Log}}(p),~
l \sim \sum_{{t}'=1}^{l'}\emph{\mbox{Log}}(p'), ~ l' \sim\emph{ \mbox{Pois}}(-r_1\ln(1-p')) .
\eeqs
\end{lem}
\vspace{-5.3mm}
\begin{proof}
Augmenting $m$ leads to
$
m \sim \sum_{{t}=1}^{l}  {\mbox{Log}}(p),~ l\sim{\mbox{Pois}}(-r\ln(1-p)) 
$.
Marginalizing out $r$ leads to $l\sim{\mbox{NB}}\left(r_1,p'\right)$. Augmenting $l$ using its compound Poisson representation leads to (\ref{eq:merge}).
\vspace{-2mm}
\end{proof}

\vspace{-3mm}
\section{Gamma-Negative Binomial Process}\label{sec:GNBP}
\vspace{-2mm}
We explore sharing the NB dispersion across groups while the probability parameters are group dependent.
We define a NB process  $X\sim\mbox{NBP}(G,p)$ as $X(A)\sim \mbox{NB}(G(A),p)$ for each $A\subset \Omega$
and construct a gamma-NB process for joint count and mixture modeling as
$ 
X_j \sim \mbox{NBP}(G,p_j),~ G\sim \mbox{GaP}(c, G_0)
$,
which can be augmented as a gamma-gamma-Poisson process as
\beqs\label{eq:HGP-PP}\small
&X_j \sim \mbox{PP}(\Lambda_j),~ ~\Lambda_j \sim \mbox{GaP}((1-p_j)/p_j, G), ~~G\sim \mbox{GaP}(c, G_0).
\eeqs
In the above PP($\cdot$) and GaP($\cdot$) represent the Poisson and gamma processes, respectively, 
as defined in Section \ref{sec:PP}. Using Lemma~\ref{lem:merge},   the gamma-NB process can also be augmented as
\vspace{-1mm}\beqs\small
&\label{eq:Aug2} X_j \sim \sum_{{t}=1}^{L_j}\mbox{Log}(p_j), ~L_j \sim \mbox{PP}(-G\ln(1-p_j)),~G\sim \mbox{GaP}(c, G_0);\\
&\label{eq:Laugment} L = \sum_j L_j \sim\sum_{{t}=1}^{L'}\mbox{Log}(p'), ~ L' \sim \mbox{PP}(-G_0\ln(1-p')),~~p' = \frac{-\sum_j\ln(1-p_j)}{c-\sum_j\ln(1-p_j)}.\vspace{-1mm}
\eeqs
These three augmentations allow us to derive a sequence of  closed-form update equations for inference 
with the gamma-NB process.
 Using the  gamma Poisson conjugacy on (\ref{eq:HGP-PP}), for each $A\subset\Omega$, we have
$
\Lambda_j(A)|G,X_j,p_j\sim \mbox{Gamma}\left(G(A)+ X_j(A),p_j\right)
$, thus the conditional posterior of 
$\Lambda_j$ is 
\vspace{-0.5mm}\beq\label{eq:post_start}
\Lambda_j|G,X_j,p_j \sim {\mbox{GaP}}\big(1/p_j,G+X_j\big).\vspace{-.5mm}
\eeq
Define $T\sim \mbox{CRTP}(X,G)$ as a CRT process that $T(A) = \sum_{\omega\in A}T(\omega),~T(\omega) \sim\mbox{CRT}(X(\omega),G(\omega))$ for each $A\subset\Omega$. Applying Lemma \ref{lem:divide} on (\ref{eq:Aug2}) and (\ref{eq:Laugment}),
we have
\beqs\small
&L_j|X_j,G \sim \mbox{CRTP}(X_j,G),~~~~L'|L,G_0 \sim \mbox{CRTP}(L,G_0).\vspace{-0.8mm}
\eeqs
If $G_0$ is a continuous base measure and $\gamma_0=G_0(\Omega)$ is finite, we have $G_0(\omega)\hspace{-1mm}\rightarrow \hspace{-0.5mm}0$ $\forall~\omega\in \Omega$  and thus
\beqs
&L'(\Omega)|L,G_0 = \sum_{\omega\in\Omega}\delta(L(\omega)>0)= \sum_{\omega\in\Omega}\delta( \sum_{j}X_j(\omega)>0)
\eeqs
which is equal to $K^{+}$, the total number of used discrete atoms; if $G_0$ is discrete as $G_0=\sum_{k=1}^K\frac{\gamma_0}{K}\delta_{\omega_k}$,
then $L'(\omega_k)=\mbox{CRT}(L(\omega_k),\frac{\gamma_0}{K})\ge 1$ if $\sum_{j}X_j(\omega_k)>0$, thus $L'(\Omega)\ge K^{+}$. In either case,
let $\gamma_0\sim{\mbox{Gamma}}(e_0,1/f_0)$,
with  the gamma Poisson conjugacy on (\ref{eq:Aug2}) and (\ref{eq:Laugment}), 
we have
\vspace{-0.0mm}\beqs \label{eq:gamma0post}
&\gamma_0|\{L'(\Omega),p'\} \sim {\mbox{Gamma}}\big( e_0 +  L'(\Omega),\frac{1}{f_0- \ln(1-p')}\big);\\
\label{eq:Gpost}
&G|G_0,\{L_j,p_j\} \sim {\mbox{GaP}} \big(c - \sum_{j} \ln( 1 - p_j), G_0 +  \sum_{j}  L_{j} \big).\vspace{-.5mm}
\eeqs
 Since the data $\{x_{ji}\}_i$ are exchangeable within group $j$, the predictive distribution of a point $X_{ji}$, conditioning on $X_j^{-i}=\{X_{jn}\}_{n: n\neq i}$ and $G$, with $\Lambda_j$ marginalized out, can be expressed as
\vspace{-1mm}\beqs\small
&X_{ji}|G,X_j^{-i}\sim \frac{\E[\Lambda_j|G,X_j^{-i}]}{\E[\Lambda_j(\Omega)|G,X_j^{-i}]} = \frac{G}{G(\Omega)+X_j(\Omega)-1}+\frac{X_j^{-i}}{G(\Omega)+X_j(\Omega)-1}.
\eeqs
\vspace{-2mm}
\subsection{Relationship with the hierarchical Dirichlet process}
\vspace{-2mm}
Using the equivalence between (\ref{eq:Pois}) and (\ref{eq:Mult})  and normalizing all the gamma processes in (\ref{eq:HGP-PP}),  denoting $\widetilde{\Lambda}_j= {\Lambda_j}/\Lambda_j(\Omega)$, $\alpha = G(\Omega)$, $ \widetilde{G} = G/\alpha$, $\gamma_0=G_0(\Omega)$ and $\widetilde{G}_0 = G_0/\gamma_0$, 
we can re-express (\ref{eq:HGP-PP}) as
\beqs\label{eq:HDP}
&X_{ji} \sim \widetilde{\Lambda}_j,~\widetilde{\Lambda}_j \sim \mbox{DP}(\alpha,  \widetilde{G}),~\alpha \sim\mbox{Gamma}(\gamma_0,1/c),~\widetilde{G}\sim \mbox{DP}(\gamma_0, \widetilde{G}_0) 
\eeqs
 which is an HDP \cite{HDP}.
Thus the normalized gamma-NB process leads to an HDP, yet we cannot return from the HDP to the gamma-NB process without modeling $X_j(\Omega)$ and $\Lambda_j(\Omega)$ as random variables.
Theoretically, they are distinct in that the gamma-NB process is a completely random measure, assigning independent random variables into any disjoint Borel sets $\{A_q\}_{1,Q}$ of $\Omega$;
whereas the HDP
is not.
Practically, the gamma-NB process can exploit conjugacy to achieve analytical conditional posteriors for all latent parameters.
The inference of the HDP is a major challenge and it is usually solved through alternative constructions such as the Chinese restaurant franchise (CRF) and stick-breaking representations \cite{HDP,VBHDP}.
In particular, without analytical conditional posteriors, the inference of concentration parameters $\alpha$ and $\gamma_0$ is nontrivial \cite{HDP,HDP-HMM} and they are often simply fixed  \cite{VBHDP}. Under the CRF metaphor $\alpha$ governs the random number of tables occupied by customers in each restaurant independently; further, if the base probability measure $\widetilde{G}_0$ is continuous, $\gamma_0$ governs the random number of dishes selected by tables of all restaurants. One may apply the data augmentation method of \cite{Escobar1995} to sample $\alpha$ and $\gamma_0$.\linebreak However, if $\widetilde{G}_0$ is discrete as $\widetilde{G}_0=\sum_{k=1}^K \frac{1}{K}\delta_{\omega_k}$, which is of practical value and becomes a continuous base measure as $K\rightarrow \infty$ \cite{ishwaran02,HDP,HDP-HMM}, then using the method of \cite{Escobar1995} to sample $\gamma_0$ is only approximately correct, which may result in a biased estimate in practice, especially if $K$ is not large enough.
By contrast, in the gamma-NB process, the shared gamma process $G$ can be analytically updated with (\ref{eq:Gpost})
and $G(\Omega)$ plays the role of $\alpha$ in the HDP, which is readily available as
\vspace{-1.0mm}\beqs \label{eq:alphaPost} \small
&G(\Omega)|G_0, \{L_j,p_j\}_{j=1,N}
\sim \mbox{Gamma}  {\Big( \gamma_0 +  \sum_j  L_{j}(\Omega), \frac{1}{c - \sum_j \ln( 1 - p_j)}\Big)}\vspace{-1.0mm}
\eeqs
and as in (\ref{eq:gamma0post}), regardless of whether the base measure is continuous, the total mass $\gamma_0$ has an analytical gamma posterior whose shape parameter is governed by  $L'(\Omega)$, with $L'(\Omega)=K^{+}$ if $G_0$ is continuous and finite and $L'(\Omega)\ge K^{+}$ if 
$G_0=\sum_{k=1}^K\frac{\gamma_0}{K}\delta_{\omega_k}$.
Equation (\ref{eq:alphaPost}) also intuitively shows how the NB probability parameters $\{p_j\}$ govern the variations among $\{\widetilde{\Lambda}_j\}$ in the gamma-NB process.
In the HDP, $p_j$ is not explicitly modeled, and since its value becomes irrelevant when taking the normalized constructions in (\ref{eq:HDP}), it is usually  treated as a nuisance parameter and perceived as $p_j=0.5$ \linebreak when needed for interpretation purpose. Fixing $p_j=0.5$ is also considered in \cite{DILN} to construct an HDP, whose group-level DPs are normalized from gamma processes with the scale parameters as $\frac{p_j}{1-p_j}=1$; it is also shown in \cite{DILN} that improved performance can be obtained for topic modeling by learning the scale parameters with a log Gaussian process prior. However, no analytical conditional posteriors are provided and Gibbs sampling is not considered as a viable option \cite{DILN}.

\vspace{-3mm}
\subsection{Augment-and-conquer inference for joint count and mixture modeling}
\vspace{-2mm}
For a finite continuous base measure, the gamma process $G\sim\mbox{GaP}(c,G_0)$ can also be defined with its L\'{e}vy measure on a product space $\mathbb{R}_{+}\times \Omega$, expressed as
$ 
\nu(dr d\omega) = r^{-1} e^{-cr} dr G_0(d\omega)
$  \cite{Wolp:Clyd:Tu:2011}. 
Since the Poisson intensity $\nu^{+}=\nu(\mathbb{R}_{+}\times\Omega)=\infty$ and $\int\int_{\mathbb{R}_{+}\times\Omega} r \nu(dr d\omega)$ is finite,  
a draw from this process 
can be expressed as
$ 
G = \sum_{k=1}^{\infty} r_k \delta_{\omega_k},~(r_k, \omega_k)\sim \pi(dr d\omega),~\pi(dr d\omega)\nu^{+} \equiv \nu(dr d\omega)$ \cite{Wolp:Clyd:Tu:2011}. 
Here we consider a discrete base measure as $G_0=\sum_{k=1}^K \frac{\gamma_0}{K}\delta_{\omega_k},~\omega_k\sim g_0(\omega_k)$, then we have $G = \sum_{k=1}^{K} r_k \delta_{\omega_k}$, $r_k\sim\mbox{Gamma}(\gamma_0/K,1/c), \omega_k\sim g_0(\omega_k)$, which becomes a draw from the gamma process with a continuous base measure as $K\rightarrow \infty$.
Let $x_{ji}\sim F(\omega_{z_{ji}})$ be observation $i$ in group $j$, linked to a mixture component $\omega_{z_{ji}}\in\Omega$ through a distribution $F$.
Denote $n_{jk}=\sum_{i=1}^{N_j}\delta({z_{ji}=k})$, we can express the gamma-NB process with the discrete base measure as
\vspace{-1mm}\beqs
&\omega_k \sim g_0(\omega_k),~N_j = \sum_{k=1}^K n_{jk},~n_{jk}\sim \mbox{Pois}(\lambda_{jk}),~\lambda_{jk}\sim \mbox{Gamma}(r_k,p_j/(1-p_j))\notag\\
&r_k\sim\mbox{Gamma}(\gamma_0/K,1/c),~p_j\sim\mbox{Beta}(a_0,b_0),~\gamma_0\sim \mbox{Gamma}(e_0,1/f_0)\label{eq:GNBP_mixture}\vspace{-1mm}
\eeqs
where marginally we have $n_{jk}\sim\mbox{NB}(r_k,p_j)$.
Using the equivalence between (1) and (2), we can equivalently express 
$N_j$ and $n_{jk}$
in the above model as
$
N_j\sim \mbox{Pois}\left(\lambda_j\right),~[n_{j1},\cdots,n_{jK}] \sim \mbox{Mult}\left(N_j;{\lambda_{j1}}/{\lambda_j},\cdots,{\lambda_{jK}}/{\lambda_j}\right)
$,
where $\lambda_j = \sum_{k=1}^K\lambda_{jk}$.
Since the data $\{x_{ji}\}_{i=1,N_j}$  are fully exchangeable, rather than drawing $[n_{j1},\cdots,n_{jK}]$ once, we may equivalently draw the index
\vspace{-.8mm}\beq \label{eq:z_ji}
z_{ji} \sim \mbox{Discrete}\left({\lambda_{j1}}/\lambda_j,\cdots,{\lambda_{jK}}/\lambda_j\right)\vspace{-1mm}
\eeq
for each $x_{ji}$ and then let $n_{jk} = \sum_{i=1}^{N_j}\delta(z_{ji}=k)$.
This provides further insights on how the seemingly disjoint problems of count and mixture modeling are united under the NB process framework.
Following (\ref{eq:post_start})-(\ref{eq:Gpost}),
the block Gibbs sampling
is straightforward to write as
\vspace{-.5mm}\beqs \label{eq:inference}\small
&p(\omega_k|-) \propto \prod_{z_{ji}=k}F(x_{ji};\omega_k) g_0(\omega_k), ~~\mbox{Pr}(z_{ji}=k|-) \propto F(x_{ji};\omega_{k})\lambda_{jk}\notag\\
&(p_j|-)\sim \mbox{Beta}\bigg(a_0+N_j,b_0+ \sum_{k}r_k\bigg),~p' = \frac{-\sum_j\ln(1-p_j)}{c-\sum_j\ln(1-p_j)},~(l_{jk}|-) \sim \mbox{CRT}(n_{jk},r_k)\notag\\
&(l'_{k}|-) \sim \mbox{CRT}(\sum_{j} l_{jk},\gamma_0/K),
~~(\gamma_0|-)  \sim  \mbox{Gamma}  \bigg( e_0  +   \sum_{k}l'_{k}, \frac{1}{f_0  -  \ln( 1 - p')} \bigg)\notag\\
&(r_k|-) \sim \mbox{Gamma} \bigg( \gamma_0/K +  \sum_{j}  l_{jk}, \frac{1}{c - \sum_{j} \ln( 1 - p_j)}\bigg), 
~~(\lambda_{jk}|-) \sim \mbox{Gamma}(r_k+n_{jk},p_j).\vspace{-2mm}
\eeqs
which has similar computational complexity as that of the direct assignment block Gibbs sampling of the CRF-HDP \cite{HDP,HDP-HMM}. If $g_0(\omega)$ is conjugate to the likelihood $F(x;\omega)$, then the posterior $p(\omega|-)$ would be analytical.
Note that when $K\rightarrow \infty$, we have $(l'_{k}|-)=\delta(\sum_jl_{jk}>0)=\delta(\sum_{j}n_{jk}>0)$.

 Using (1) and (2) and normalizing the gamma distributions, (\ref{eq:GNBP_mixture}) can be re-expressed as
\beqs
&
{z_{ji}}\sim \mbox{Discrete}(\tilde{\lambdav}_j), ~\tilde{\lambdav}_j\sim \mbox{Dir}(\alpha \tilde{\rv}),~\alpha\sim\mbox{Gamma}(\gamma_0,1/c),~\tilde{\rv} \sim \mbox{Dir}({\gamma_0}/{K},\cdots,{\gamma_0}/{K})
\eeqs
which loses the count modeling ability and becomes a finite representation of the HDP, the inference of which is not conjugate and has to be solved under alternative representations 
\cite{HDP,HDP-HMM}. 
This also implies that by using the Dirichlet process as the foundation, traditional mixture modeling 
may discard useful count information from the beginning.

\vspace{-3mm}
\section{The Negative Binomial Process Family and Related Algorithms}
\vspace{-3mm}
The gamma-NB process shares the NB dispersion across groups. Since the NB distribution has two adjustable parameters, we may explore alternative ideas, with the NB probability measure shared across groups as in \cite{NBPJordan}, or with both the dispersion and probability measures shared as in \cite{BNBP_PFA}. These constructions are distinct from both the gamma-NB process and HDP in that $\Lambda_j$ has space dependent scales, and thus its normalization $\widetilde{\Lambda}_j= {\Lambda_j}/\Lambda_j(\Omega)$ no longer follows a Dirichlet process.

 It is natural to let the probability measure be drawn from a beta process \cite{Hjort,JordanBP}, which can be defined by its L\'{e}vy  measure on a product space $[0,1]\times \Omega$ as
 $
 \nu(dpd\omega)=cp^{-1}(1-p)^{c-1}dp B_0(d\omega).
 $
  A draw from the beta process $B\sim\mbox{BP}(c,B_0)$ with concentration parameter $c$ and base measure $B_0$ can be expressed as
$
B = \sum_{k=1}^\infty p_k \delta_{\omega_k}. 
$ A beta-NB process \cite{BNBP_PFA,NBPJordan} can be constructed by letting $X_j\sim\mbox{NBP}(r_j,B)$, with a random draw
expressed as
$
X_j = \sum_{k=1}^\infty n_{jk}\delta_{\omega_k}, ~n_{jk}\sim \mbox{NB}(r_j,p_k).
$ Under this construction, the NB probability measure is shared  and the NB dispersion parameters are group dependent. As in \cite{BNBP_PFA}, we may also consider a marked-beta-NB\footnote{We may also consider a beta marked-gamma-NB process, whose performance is found to be very similar.} process that both the NB probability and dispersion measures are shared, in which each point of the beta process is marked with an independent gamma random variable. Thus a draw from the marked-beta process becomes $(R,B) = \sum_{k=1}^\infty (r_k,p_k) \delta_{\omega_k}$, and the NB process $X_j\sim \mbox{NBP}(R,B)$ becomes $
X_j = \sum_{k=1}^\infty n_{jk}\delta_{\omega_k}, ~n_{jk}\sim \mbox{NB}(r_k,p_k).
$
Since the beta and NB processes are conjugate, the posterior of $B$ is tractable, as shown in \cite{BNBP_PFA,NBPJordan}. If it is believed that there are excessive number of zeros, governed by  a process other than the NB process, we may introduce a zero inflated NB process as $X_j \sim \mbox{NBP}(RZ_j,p_j)$, where $Z_j\sim \mbox{BeP}(B)$ is drawn from the Bernoulli process \cite{JordanBP} and $(R,B) = \sum_{k=1}^\infty (r_k,\pi_k)\delta_{\omega_k}$ is drawn from a marked-beta process, thus $n_{jk}\sim \mbox{NB}(r_k b_{jk}, p_j), ~b_{jk}  = \mbox{Bernoulli}(\pi_k)$. This construction can be linked to the model in \cite{FocusTopic} with appropriate normalization, with advantages that there is no need to fix $p_j=0.5$ and the inference is fully tractable. The zero inflated construction can also be linked to models for real valued data using the Indian buffet process (IBP) or beta-Bernoulli process spike-and-slab prior \cite{IBP,BPFA2012,dHBP,DictTopic}.

\vspace{-3mm}
\subsection{Related Algorithms}
\vspace{-2mm}
To show how the NB processes can be diversely constructed and to make connections to previous parametric and nonparametric mixture models, we show in Table \ref{Tab:Relationships} a variety of NB processes, which differ on how the dispersion and probability measures are shared.
 For a deeper understanding on how the counts are modeled, we also show in Table \ref{Tab:Relationships} both the VMR and ODL implied by these settings.
We consider topic modeling of a document corpus, a typical example of mixture modeling of grouped data, where each a-bag-of-words document constitutes a group, each word is an exchangeable group member, and $F(x_{ji};\omega_k)$ is simply the probability of word $x_{ji}$ in topic $\omega_k$.

We consider six differently constructed NB processes in  Table \ref{Tab:Relationships}:
(\emph{i}) Related to latent Dirichlet  allocation (LDA) \cite{LDA} 
and Dirichlet Poisson factor analysis (Dir-PFA) \cite{BNBP_PFA}, the NB-LDA is also a parametric topic model that requires tuning the number of topics. However, it uses a document dependent $r_j$ and $p_j$ to automatically learn the smoothing of the gamma distributed topic weights, and it lets $r_j\sim\mbox{Gamma}(\gamma_0,1/c),~\gamma_0\sim\mbox{Gamma}(e_0,1/f_0)$ to share statistical strength between documents, with closed-form Gibbs sampling inference. Thus even the most basic parametric LDA topic model can be improved under the NB count modeling framework.
(\emph{ii}) The NB-HDP model  is related to the HDP \cite{HDP}, and since $p_j$ is an irrelevant parameter in the HDP due to normalization, we set it in the NB-HDP as 0.5, the usually perceived value before normalization. The NB-HDP model is comparable to the DILN-HDP \cite{DILN} that constructs the group-level DPs with normalized gamma processes, whose scale parameters are also set as one. 
(\emph{iii}) The NB-FTM model introduces an additional\linebreak beta-Bernoulli process under the NB process framework to explicitly model zero counts.
It is the same as the sparse-gamma-gamma-PFA (S$\gamma\Gamma$-PFA) in \cite{BNBP_PFA} and is comparable to the focused topic model (FTM) \cite{FocusTopic}, which is constructed from the IBP compound DP. Nevertheless, they apply about the same likelihoods and priors for inference.
The Zero-Inflated-NB process improves over them by allowing $p_j$ to be inferred, which generally yields better data fitting.
(\emph{iv}) The Gamma-NB process  explores the idea that the  dispersion measure is shared across groups, and it improves over the NB-HDP by allowing the learning of $p_j$. It reduces to the HDP \cite{HDP} by normalizing both the group-level and the shared gamma processes.
(\emph{v}) The Beta-NB process  explores sharing the probability measure  across groups, and it improves over the beta negative binomial process (BNBP) proposed in \cite{NBPJordan}, allowing inference of $r_j$.
(\emph{vi}) The Marked-Beta-NB process is comparable to the BNBP proposed in \cite{BNBP_PFA}, with the distinction that it allows analytical update of $r_k$. The constructions and inference of various NB processes and related algorithms in Table \ref{Tab:Relationships} all follow the formulas in (\ref{eq:GNBP_mixture}) and (\ref{eq:inference}), respectively, with additional details presented in the supplementary material.

Note that as shown in \cite{BNBP_PFA}, NB process topic models can also be considered as factor analysis of the term-document count matrix under the Poisson likelihood, with $\omega_k$ as the $k$th factor loading that sums to one and $\lambda_{jk}$ as the factor score, which can be further linked to nonnegative matrix factorization \cite{NMF} and a gamma Poisson factor model \cite{CannyGaP}. If except for proportions $\tilde{\lambdav}_j$ and $\tilde{r}$, the absolute values, e.g., $\lambda_{jk}$, $r_k$ and $p_k$, are also of interest, 
then the NB process based joint count and mixture models would apparently be more appropriate than the HDP based mixture models.

\begin{table}\caption{\small A variety of negative binomial processes are constructed with distinct sharing mechanisms, reflected with which parameters from $r_k$, $r_j$, $p_k$, $p_j$ and $\pi_k$ ($b_{jk}$) are inferred (indicated by a check-mark~$\checkmark$), and the implied VMR and ODL
for counts $\{n_{jk}\}_{j,k}$. They are applied for topic modeling of a document corpus, a typical example of mixture modeling of grouped data. Related  algorithms are shown in the last column.}
\centering
{\footnotesize
\begin{tabular}{|c|c|c|c|c|c|c|c|c|}
  \hline
  Algorithms                     & $r_k$ & $r_j$ & $p_k$ & $p_j$ & $\pi_k$ & VMR & ODL & Related Algorithms\\ \hline
  NB-LDA             &  & $\checkmark$ &  & $\checkmark$ & & $(1-p_j)^{-1}$ & $r_j^{-1}$ &LDA \cite{LDA}, Dir-PFA \cite{BNBP_PFA} \\\hline
  NB-HDP                   & $\checkmark$ & &  & $0.5$ & & $2$ & $r_k^{-1}$ &HDP\cite{HDP}, DILN-HDP \cite{DILN} \\ \hline
  NB-FTM       & $\checkmark$ &  & & $0.5$ & $\checkmark$ & $2$ & $(r_k)^{-1}b_{jk}$& FTM \cite{FocusTopic}, S$\gamma\Gamma$-PFA \cite{BNBP_PFA}\\  \hline
  Beta-NB        &  & $\checkmark$  & $\checkmark$ &  & & $(1-p_k)^{-1}$ & $r_j^{-1}$ &  BNBP \cite{BNBP_PFA}, BNBP \cite{NBPJordan}\\ \hline
  Gamma-NB                   & $\checkmark$ & &  & $\checkmark$ & & $(1-p_j)^{-1}$ & $r_k^{-1}$ & CRF-HDP \cite{HDP,HDP-HMM} \\ \hline
  Marked-Beta-NB       & $\checkmark$ &  & $\checkmark$ &  & & $(1-p_k)^{-1}$ & $r_k^{-1}$ & BNBP \cite{BNBP_PFA}   \\
  \hline
\end{tabular}\label{Tab:Relationships}
}\vspace{-4mm}
\end{table}
\vspace{-3mm}
\section{Example Results}
\vspace{-3mm}

Motivated by  Table \ref{Tab:Relationships}, we consider topic modeling using a variety of  NB processes, which differ on which parameters are learned and consequently how the VMR and ODL of the latent counts $\{n_{jk}\}_{j,k}$ are modeled. We compare them with LDA \cite{LDA,FindSciTopic} and CRF-HDP \cite{HDP,HDP-HMM}. For fair comparison, they are all implemented with block Gibbs sampling using a discrete base measure with $K$ atoms, and for the first fifty iterations, the Gamma-NB process with $r_k\equiv 50/K$ and $p_j\equiv 0.5$ is used for initialization. 
For LDA and NB-LDA, we search $K$ for optimal performance and for the other models, we set $K=400$ as an upper-bound.
We set the parameters as $c=1$, $\eta=0.05$ and $a_0=b_0=e_0=f_0=0.01$.
For LDA, we set the topic proportion Dirichlet smoothing parameter as $50/K$, following the topic model toolbox$^2$ provided for \cite{FindSciTopic}. We consider 2500 Gibbs sampling iterations, with the last 1500 samples collected.
Under the NB processes, each word $x_{ji}$ would be assigned to a topic $k$ based on both $F(x_{ji};\omega_k)$ and the topic weights $\{\lambda_{jk}\}_{k=1,K}$; each topic is drawn from a Dirichlet  base measure as $\omega_k \sim \mbox{Dir}(\eta,\cdots,\eta) \in \mathbb{Re}^{V}$, where $V$ is the number of unique terms in the vocabulary and $\eta$ is a smoothing parameter.
Let $v_{ji}$ denote the location of word $x_{ji}$ in the vocabulary, then we have
$ 
(\omega_k|-)\sim \mbox{Dir}\big(\eta+ \sum_{j}\sum_{i}\delta(z_{ji}=k,v_{ji}=1),  \cdots,\eta+\sum_{j}\sum_{i}\delta(z_{ji}=k,v_{ji}=V)\big).
$
We consider the Psychological Review\footnote{\label{footnote}http://psiexp.ss.uci.edu/research/programs$\_$data/toolbox.htm} corpus, restricting the vocabulary to terms that occur in
five or more documents. The corpus includes 1281 abstracts from 1967 to 2003, with 2,566 unique terms and 71,279 total word counts.
We randomly select $20\%$, $40\%$, $60\%$ or $80\%$ of the words from each document  to learn a document dependent probability for each term $v$ as $f_{jv} = {\sum_{s=1}^S\sum_{k=1}^K\omega^{(s)}_{vk}\lambda^{(s)}_{jk}}\big/{\sum_{s=1}^S\sum_{v=1}^V\sum_{k=1}^K\omega^{(s)}_{vk}\lambda^{(s)}_{jk}}$, where $\omega_{vk}$ is the probability of term $v$ in topic $k$ and $S$ is the total number of collected samples. We use $\{f_{jv}\}_{j,v}$ to calculate the per-word perplexity on the held-out words as in \cite{BNBP_PFA}.
The final results are averaged from five random training/testing partitions. 
Note that the perplexity per test word is the fair metric to compare topic models.
However, when the actual Poisson rates or distribution parameters for counts instead of the mixture proportions are of interest, it is obvious that a NB process based joint count and mixture model would be more appropriate than an HDP based mixture model.

Figure \ref{fig:Perplexity} compares the performance of various algorithms. The Marked-Beta-NB process has the best performance, closely followed by the Gamma-NB process, CRF-HDP and Beta-NB process. With an appropriate $K$, the parametric NB-LDA may outperform the nonparametric NB-HDP and NB-FTM as the training data percentage increases, somewhat unexpected but very intuitive results, showing that even by learning both the NB dispersion and probability parameters $r_j$ and $p_j$ in a document dependent manner, we may get better data fitting than using nonparametric models that share the NB dispersion parameters $r_k$ across documents, but fix the NB probability parameters.

\begin{figure}[!tb]
\begin{center}
\includegraphics[width=108mm]{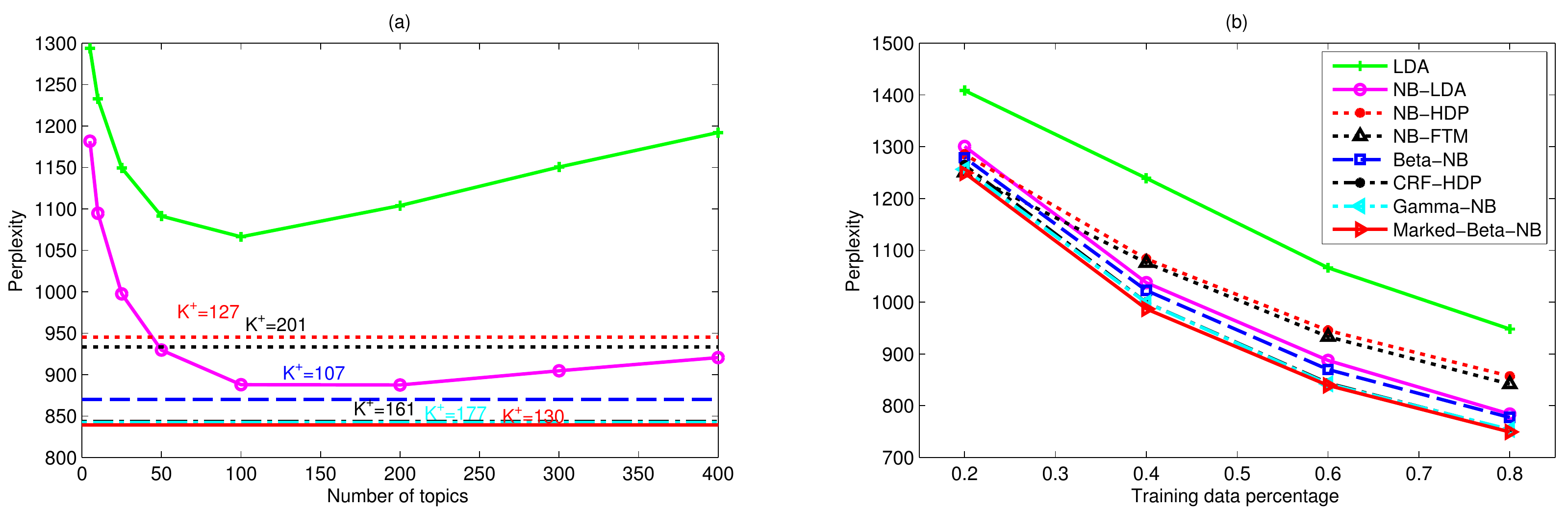}
\end{center}
\vspace{-4.0mm}
\caption{\small \label{fig:Perplexity}
Comparison of per-word perplexities on the held-out words between various algorithms.  (a) With $60\%$ of the words in each document used for training, the performance varies as a function of $K$ in both LDA and NB-LDA, which are parametric models, whereas the NB-HDP, NB-FTM, Beta-NB, CRF-HDP, Gamma-NB and Marked-Beta-NB all infer the number of active topics, which are 127, 201, 107, 161, 177 and 130, respectively, according to the last Gibbs sampling iteration. (b) Per-word perplexities of various models as a function of the percentage of words in each document used for training. The results of the LDA and NB-LDA are shown with the best settings of $K$ under each training/testing partition.
\vspace{-4.0mm}
}
\end{figure}

Figure \ref{fig:MakeSense} shows the learned model parameters by various algorithms under the NB process framework, revealing distinct sharing mechanisms and model properties. When $(r_j,p_j)$ is used, as in the NB-LDA, different documents are weakly coupled with $r_j\sim\mbox{Gamma}(\gamma_0,1/c)$,
and the modeling results show that  a typical document in this corpus usually has a small $r_j$ and a large $p_j$, thus a large ODL and a large VMR, indicating highly overdispersed counts on its topic usage.
When $(r_j,p_k)$ is used to model the latent counts $\{n_{jk}\}_{j,k}$, as in the Beta-NB process, the transition between active and non-active topics is very sharp that $p_k$ is either close to one or close to zero. That is because $p_k$ controls the mean as $\E[\sum_{j}n_{jk}]=p_k/(1-p_k)\sum_{j}r_j$ and the VMR as $(1-p_k)^{-1}$ on topic $k$, thus a popular topic must also have large $p_k$ and thus large overdispersion measured by the VMR; since the counts $\{n_{jk}\}_{j}$ are  usually overdispersed, particularly  true in this corpus, a middle range $p_k$ indicating an appreciable mean and small overdispersion is not favored by the model and thus is rarely observed. When $(r_k,p_j)$ is used, as in the Gamma-NB process, the transition is much smoother that $r_k$ gradually decreases. The reason is that  $r_k$ controls the mean as $\E[\sum_{j}n_{jk}]=r_k\sum_{j}p_j/(1-p_j)$ and the ODL $r_k^{-1}$ on topic $k$, thus popular topics must also have large $r_k$ and thus small overdispersion measured by the ODL, and unpopular topics are modeled with small $r_k$ and thus large overdispersion, allowing rarely and lightly used topics. Therefore, we can expect that $(r_k,p_j)$ would allow more topics than $(r_j,p_k)$,  as confirmed in Figure \ref{fig:Perplexity} (a) that the Gamma-NB process learns 177 active topics, significantly  more than the 107 ones of the Beta-NB process. With these analysis, we can conclude that the mean and the amount of overdispersion (measure by the VMR or ODL) for the usage of topic $k$ is positively correlated under $(r_j,p_k)$ and negatively correlated under $(r_k,p_j)$.

When $(r_k,p_k)$ is used, as in the Marked-Beta-NB process, more diverse combinations of mean and overdispersion would be allowed as both $r_k$ and $p_k$ are now responsible for the mean $\E[\sum_{j}n_{jk}]=Jr_kp_k/(1-p_k)$. For example, there could be not only large mean and small overdispersion (large $r_k$ and small $p_k$), but also large mean and large overdispersion (small $r_k$ and large $p_k$). Thus $(r_k,p_k)$ may combine the advantages of using only $r_k$ or $p_k$ to model topic $k$, as confirmed by the superior performance of the Marked-Beta-NB over the Beta-NB and Gamma-NB processes.  When $(r_k, \pi_k)$ is used, as in the NB-FTM model, our results show that we usually have a small $\pi_k$ and a large $r_k$, indicating topic $k$ is sparsely used across the documents but once it is used, the amount of variation on usage is small. This modeling properties might be helpful when there are excessive number of zeros which might not be well modeled by the NB process alone. In our experiments, we find the more direct approaches of using $p_k$ or $p_j$ generally yield better results, but this might not be the case when excessive number of zeros are better explained with the underlying beta-Bernoulli or IBP processes, e.g., when the training words are scarce.

It is also interesting to compare the Gamma-NB and NB-HDP. From a mixture-modeling viewpoint, fixing $p_j = 0.5$ is natural as $p_j$ becomes irrelevant after normalization. However, from a count modeling viewpoint, this would make a restrictive assumption that each count vector $\{n_{jk}\}_{k=1,K}$ has the same VMR of 2, and the experimental results in Figure \ref{fig:Perplexity} confirm the importance of learning $p_j$ together with $r_k$.
It is also interesting to examine (\ref{eq:alphaPost}), which can be viewed as the  concentration parameter $\alpha$  in the HDP, allowing the adjustment of $p_j$ would allow a more flexible model assumption on the amount of variations between the topic proportions, and thus potentially better data fitting.  

\begin{figure}[!tb]
\begin{center}
\includegraphics[width=110mm]{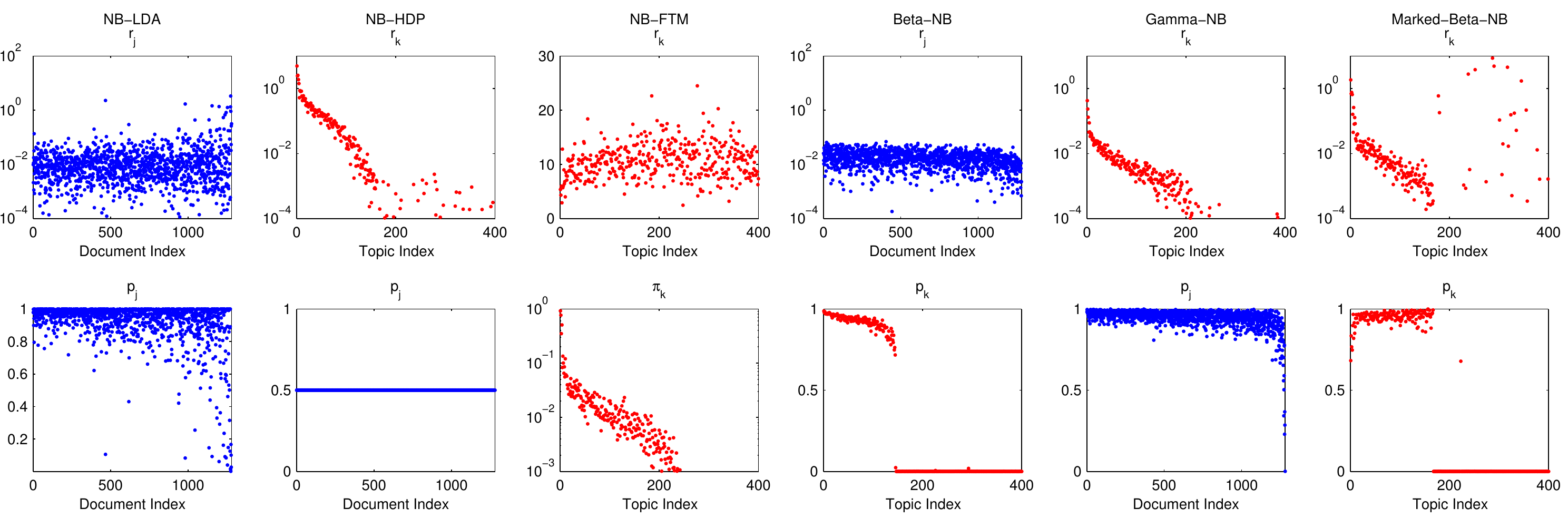}
\end{center}
\vspace{-4.0mm}
\caption{\small \label{fig:MakeSense}
Distinct sharing mechanisms and model properties are evident between various NB processes, by comparing their inferred parameters. Note that the transition between active and non-active topics is very sharp when $p_k$ is used and much smoother when $r_k$ is used.  Both the documents and topics are ordered in a decreasing order based on the number of words associated with each of them. These results are based on the last Gibbs sampling iteration.
The values are shown in either linear or log scales for convenient visualization.
\vspace{-4.5mm}
}
\end{figure}

\vspace{-3mm}
\section{Conclusions}
\vspace{-3mm}
We propose a variety of negative binomial (NB) processes to jointly model counts across groups,
which can be naturally applied for mixture modeling of grouped data. The proposed NB processes are completely random measures that they assign independent random variables to disjoint Borel sets of the measure space, as opposed to the hierarchical Dirichlet process (HDP) whose measures on disjoint Borel sets are negatively correlated. We discover augment-and-conquer inference methods  that by ``augmenting'' a NB process into both the gamma-Poisson and compound Poisson representations, we are able to ``conquer'' the unification of count and mixture modeling, the analysis of fundamental model properties and the derivation of efficient  Gibbs sampling inference.
We demonstrate that the gamma-NB process, which shares the NB dispersion measure across groups,  can be normalized to produce the HDP and we show in detail its theoretical, structural and computational advantages over the HDP. We examine the distinct sharing mechanisms and model properties of various NB processes, with connections to existing algorithms, with experimental results on topic modeling showing the importance of modeling both the NB dispersion and probability parameters.

\vspace{-2mm}
\subsubsection*{Acknowledgments}
\vspace{-2mm}
The research reported here was supported by ARO, DOE, NGA, and ONR, and by DARPA under the MSEE and HIST programs.
\small
\bibliography{NBdistribution,PGFA_NIPS,HBP_AISTATS2011,GammaNBP}
\bibliographystyle{unsrt}

%
%
%
%

\newpage
\normalsize

\appendix{Augment-and-Conquer Negative Binomial Processes: Supplementary Material}

\section{Generating a CRT random variable}
\begin{lem}\label{lem:BernoulliSum} A CRT random variable $l\sim\emph{\mbox{CRT}}(m,r)$ can be generated with the summation of independent Bernoulli random variables as
\beq
l =\sum_{n=1}^m b_n,~b_n\sim \mbox{\emph{Bernoulli}}\left(\frac{r}{n-1+r}\right).
\eeq
\end{lem}

\begin{proof}
Since $l$ is the summation of independent Bernoulli random variables, its PGF becomes
\begin{align}\notag
C_L(z)  = \prod_{n=1}^m \left( \frac{n-1}{n-1+r} + \frac{r}{n-1+r}z \right)
= {\frac{\Gamma(r)}{\Gamma(m+r)}} \sum_{k=0}^m|s(m,k)| (rz)^k.
\end{align}
Thus we have $
f_L(l|m,r) = \frac{C^{(l)}_L(0)}{l!} = {\frac{\Gamma(r)}{\Gamma(m+r)}}|s(m,l)| r^l, ~~ l=0,1,\cdots,m.
$ 
\end{proof}

\section{Dir-PFA and LDA}

The Dirichlet Poisson factor analysis (Dir-PFA) model [5] is constructed as
\begin{align}
&x_{ji} \sim F(\omega_{z_{ji}}), ~~\omega_k \sim \mbox{Dir}(\eta,\cdots,\eta)\notag\notag\\
&N_j = \sum_{k=1}^K n_{jk},~~n_{jk}\sim \mbox{Pois}(\tilde{\lambda}_{jk}),~~\tilde{\lambdav}_{j}\sim \mbox{Dir}(50/K,\cdots,50/K)
\end{align}
where $\eta$ is the Dirichlet smoothing parameter for the topic's distribution over the vocabulary, $n_{jk}=\sum_{i=1}^{N_j}\delta({z_{ji}=k})$, and the data likelihood $F(x_{ji};\omega_{k})$ in topic modeling is $\omega_{v_{ji}k}$, the probability of the $i$th word in $j$th document under topic $\omega_{k}$.

The Dir-PFA has the same block Gibbs sampling as LDA [34], expressed as
\begin{align}
&\mbox{Pr}(z_{ji}=k|-) \propto F(x_{ji};\omega_{k})\tilde{\lambda}_{jk}\notag\\
&(\omega_k|-)\sim \mbox{Dir}\left(\eta+ \sum_{j=1}^J\sum_{i=1}^{N_j}\delta(z_{ji}=k,v_{ji}=1),  \cdots,\eta+\sum_{j=1}^J\sum_{i=1}^{N_j}\delta(z_{ji}=k,v_{ji}=V)\right)\notag\\
&(\tilde{\lambdav}_j|-)\sim \mbox{Dir}\left( 50/K + n_{j1},\cdots,50/K + n_{jK}\right).
\end{align}

\section{CRF-HDP}

The CRF-HDP  model [7, 26] is constructed as
\begin{align}
&x_{ji} \sim F(\omega_{z_{ji}}), ~~\omega_k \sim \mbox{Dir}(\eta,\cdots,\eta),~~{z_{ji}}\sim \mbox{Discrete}(\tilde{\lambdav}_j)\notag\\
&\tilde{\lambdav}_j\sim \mbox{Dir}(\alpha \tilde{\rv}),~~\alpha\sim\mbox{Gamma}(a_0,1/b_0),~~\tilde{\rv} \sim \mbox{Dir}({\gamma_0}/{K},\cdots,{\gamma_0}/{K}). \end{align}

Under the CRF metaphor, denote $n_{jk}$ as the number of customers eating dish $k$ in restaurant $j$  and $l_{jk}$ as the number of tables serving dish $k$ in restaurant $j$, the direct assignment block Gibbs sampling can be expressed as
\begin{align}
&\mbox{Pr}(z_{ji}=k|-) \propto F(x_{ji};\omega_{k})\tilde{\lambda}_{jk}\notag\\
&(l_{jk}|-) \sim \mbox{CRT}(n_{jk},\alpha\tilde{r}_k),~~w_j \sim \mbox{Beta}(\alpha+1,N_j),~~s_j \sim \mbox{Bernoulli}\left(\frac{N_j}{N_j+\alpha}\right)\notag\\
&\alpha\sim\mbox{Gamma}\left(a_0+\sum_{j=1}^J\sum_{k=1}^K l_{jk}- \sum_{j=1}^J  s_{j},\frac{1}{b_0-\sum_{j}\ln  w_j}\right)\notag\\
&(\tilde\rv|-) \sim \mbox{Dir} \left( \gamma_0/K +  \sum_{j=1}^J  l_{j1},\cdots,\gamma_0/K +  \sum_{j=1}^J  l_{jK}\right)\notag\\
&(\tilde{\lambdav}_j|-)\sim \mbox{Dir}\left( \alpha \tilde{r}_1 + n_{j1},\cdots,\alpha \tilde{r}_K + n_{jK}\right)\notag\\
&(\omega_k|-)\sim \mbox{Dir}\left(\eta+ \sum_{j=1}^J\sum_{i=1}^{N_j}\delta(z_{ji}=k,v_{ji}=1),  \cdots,\eta+\sum_{j=1}^J\sum_{i=1}^{N_j}\delta(z_{ji}=k,v_{ji}=V)\right).
\end{align}
When $K\rightarrow\infty$, the concentration parameter $\gamma_0$ can be sampled as
\begin{align}
w_0 &\sim \mbox{Beta}\left(\gamma_0+1,\sum_{j=1}^J\sum_{k=1}^\infty l_{jk}\right),~~\pi_0 = \frac{e_0+K^{+}-1}{e_0+K^{+}-1+(f_0-\ln w_0)\sum_{j=1}^J\sum_{k=1}^\infty l_{jk}}\notag\\
\gamma_0 &\sim \pi_0 \mbox{Gamma}\left(e_0+K^{+},\frac{1}{f_0-\ln w_0}\right) + (1-\pi_0) \mbox{Gamma}\left(e_0+K^{+}-1,\frac{1}{f_0-\ln w_0}\right)
\end{align}
where $K^+$ is the number of used atoms. Since it is infeasible in practice to let $K\rightarrow \infty$, directly using this method to sample $\gamma_0$ is only approximately correct, which may result in a biased estimate especially if $K$ is not set large enough. Thus in the experiments, $\gamma_0$ is not sampled and is fixed as one. Note that for  implementation convenience, it is also common to fix the concentration parameter $\alpha$ as one [25]. We find through experiments that learning this parameter usually results in obviously lower per-word perplexity for held out words, thus we allow the learning of $\alpha$ using the data augmentation method proposed in [7], which is modified from the one proposed in [24].

\section{NB-LDA}

The NB-LDA model is constructed as
\begin{align}
&x_{ji} \sim F(\omega_{z_{ji}}), ~~\omega_k \sim \mbox{Dir}(\eta,\cdots,\eta)\notag\notag\\
&N_j = \sum_{k=1}^K n_{jk},~~n_{jk}\sim \mbox{Pois}(\lambda_{jk}),~~\lambda_{jk}\sim \mbox{Gamma}(r_j,p_j/(1-p_j))\notag\notag\\
&r_j\sim\mbox{Gamma}(\gamma_0,1/c),~~ p_j\sim\mbox{Beta}(a_0,b_0),~~\gamma_0\sim \mbox{Gamma}(e_0,1/f_0)
\end{align}
Note that letting $r_j\sim\mbox{Gamma}(\gamma_0,1/c),~\gamma_0\sim \mbox{Gamma}(e_0,1/f_0)$ allows different documents to share statistical strength for inferring their NB dispersion parameters.

The block Gibbs sampling can be expressed as
\begin{align}
&\mbox{Pr}(z_{ji}=k|-) \propto F(x_{ji};\omega_{k})\lambda_{jk}\notag\\
&(p_j|-)\sim \mbox{Beta}\left(a_0+N_j,b_0+ Kr_j\right),~~p'_j=\frac{-K\ln(1-p_j)}{c-K\ln(1-p_j)}\notag\\
& (l_{jk}|-) \sim\ \mbox{CRT}(n_{jk},r_j),~~ l'_j \sim\mbox{CRT}(\sum_{k=1}^Kl_{jk},\gamma_0),~~\gamma_0\sim\mbox{Gamma}\left(e_0+\sum_{j=1}^Jl'_j,\frac{1}{f_0-\sum_{j=1}^J\ln(1-p'_j)}\right)\notag\\
&(r_j|-) \sim \mbox{Gamma} \left( \gamma_0+  \sum_{k=1}^K  l_{jk}, \frac{1}{c - K \ln( 1 - p_j)}\right),~~(\lambda_{jk}|-) \sim \mbox{Gamma}(r_j+n_{jk},p_j)\notag\\
&(\omega_k|-)\sim \mbox{Dir}\left(\eta+ \sum_{j=1}^J\sum_{i=1}^{N_j}\delta(z_{ji}=k,v_{ji}=1),  \cdots,\eta+\sum_{j=1}^J\sum_{i=1}^{N_j}\delta(z_{ji}=k,v_{ji}=V)\right).
\end{align}

\section{NB-HDP}

The NB-HDP model is a special case of the Gamma-NB process model with $p_j=0.5$.
The hierarchical model and inference for the Gamma-NB process are shown in (16) and (18) of the main paper, respectively.

\section{NB-FTM}
The NB-FTM model is a special case of zero-inflated NB process with $p_j=0.5$, which is constructed as
\begin{align}
x_{ji} &\sim F(\omega_{z_{ji}}), ~~\omega_k \sim \mbox{Dir}(\eta,\cdots,\eta)\notag\\
N_j &= \sum_{k=1}^K n_{jk},~~ n_{jk}\sim \mbox{Pois}(\lambda_{jk})\notag\\
\lambda_{jk} &\sim \mbox{Gamma}(r_k b_{jk} , 0.5/(1-0.5))\notag\\
r_k&\sim\mbox{Gamma}(\gamma_0, 1/c),~
\gamma_0\sim\mbox{Gamma}(e_0, 1/f_0)\notag\\
b_{jk}&\sim\mbox{Bernoulli}(\pi_k),~~ \pi_k\sim\mbox{Beta}(c/K,c(1-1/K)).
\end{align}
The block Gibbs sampling can be expressed as
\begin{align}
&\mbox{Pr}(z_{ji}=k|-) \propto F(x_{ji};\omega_{k})\lambda_{jk}\notag\\
&b_{jk} \sim \delta(n_{jk}=0)\mbox{Bernoulli}\left( \frac{\pi_k(1-0.5)^{r_k}}{\pi_k(1-0.5)^{r_k} + (1-\pi_k)} \right) + \delta(n_{jk}>0)\notag\\
&\pi_{k} \sim \mbox{Beta}\bigg(c/K + \sum_{j=1}^J b_{jk},c(1-1/K)+J- \sum_{j=1}^J b_{jk}\bigg),~p_k' = \frac{-\sum_{j}b_{jk}\ln(1-0.5)}{c-\sum_{j}b_{jk}\ln(1-0.5)}\notag\\
&(l_{jk}|-) \sim \mbox{CRT}(n_{jk},r_k b_{jk}),~(l'_{k}|-) \sim \mbox{CRT}\left(\sum_{j=1}^Jl_{jk},\gamma_0\right)\notag\\
&(\gamma_0|-) \sim \mbox{Gamma} \left( e_0 +  \sum_{k=1}^K l'_{k}, \frac{1}{f_0 - \sum_{k=1}^K\ln( 1 - p_k')}\right)\notag\\
&(r_k|-) \sim \mbox{Gamma} \left( \gamma_0 +  \sum_{j=1}^Jl_{jk}, \frac{1}{c - \sum_{j=1}^J b_{jk} \ln( 1 - 0.5)}\right)\notag\\
&(\lambda_{jk}|-) \sim \mbox{Gamma}(r_k b_{jk} +n_{jk},0.5) \notag\\
&(\omega_k|-)\sim \mbox{Dir}\left(\eta+ \sum_{j=1}^J\sum_{i=1}^{N_j}\delta(z_{ji}=k,v_{ji}=1),  \cdots,\eta+\sum_{j=1}^J\sum_{i=1}^{N_j}\delta(z_{ji}=k,v_{ji}=V)\right).
\end{align}

\section{Beta-NB}
The beta-NB process model is constructed as
\begin{align}
x_{ji} &\sim F(\omega_{z_{ji}}), ~~\omega_k \sim \mbox{Dir}(\eta,\cdots,\eta)\notag\\
N_j &= \sum_{k=1}^K n_{jk},~~n_{jk}\sim \mbox{Pois}(\lambda_{jk}),~~
\lambda_{jk}\sim \mbox{Gamma}(r_j,p_k/(1-p_k))\notag\\
r_j&\sim\mbox{Gamma}(e_0,1/f_0),~~p_k\sim\mbox{Beta}(c/K,c(1-K))
\end{align}
The block Gibbs sampling can be expressed as
\begin{align}
&\mbox{Pr}(z_{ji}=k|-) \propto F(x_{ji};\omega_{k})\lambda_{jk}\notag\\
&(p_k|-)\sim \mbox{Beta}\left(c/K+\sum_{j=1}^Jn_{jk},c(1-1/K)+ \sum_{j=1}^Jr_j\right),
~~l_{jk} \sim \mbox{CRT}(n_{jk},r_j)\notag\\
&(r_j|-) \sim \hspace{-0.0mm}\mbox{Gamma} \left( e_0+  \sum_{k=1}^K  l_{jk}, \frac{1}{f_0 - \sum_{k=1}^K \ln( 1 - p_k)}\right)\notag\\
&(\lambda_{jk}|-) \sim \mbox{Gamma}(r_j+n_{jk},p_k)\notag\\
&(\omega_k|-)\sim \mbox{Dir}\left(\eta+ \sum_{j=1}^J\sum_{i=1}^{N_j}\delta(z_{ji}=k,v_{ji}=1),  \cdots,\eta+\sum_{j=1}^J\sum_{i=1}^{N_j}\delta(z_{ji}=k,v_{ji}=V)\right).
\end{align}

\section{Marked-Beta-NB}
The Marked-Beta-NB process model is constructed as
\begin{align}
x_{ji} &\sim F(\omega_{z_{ji}}), ~~\omega_k \sim \mbox{Dir}(\eta,\cdots,\eta)\notag\\
N_j &= \sum_{k=1}^K n_{jk},~~n_{jk}\sim \mbox{Pois}(\lambda_{jk}),~~\lambda_{jk}\sim \mbox{Gamma}(r_k,p_k/(1-p_k))\notag\\
r_k&\sim\mbox{Gamma}(e_0,1/f_0),~~ p_k\sim\mbox{Beta}(c/K,c(1-K)) 
\end{align}
The block Gibbs sampling can be expressed as
\begin{align}
&\mbox{Pr}(z_{ji}=k|-) \propto F(x_{ji};\omega_{k})\lambda_{jk}\notag\\
&p_k\sim \mbox{Beta}\left(c/K+\sum_{j=1}^Jn_{jk},c(1-1/K)+ Jr_k\right),~~l_{jk} \sim\mbox{CRT}(n_{jk},r_k) \notag\\
&(r_k|-) \sim \mbox{Gamma} \left( e_0+  \sum_{j=1}^J  l_{jk}, \frac{1}{f_0- J \ln( 1 - p_k)}\right)\notag\\
&(\lambda_{jk}|-) \sim \mbox{Gamma}(r_k+n_{jk},p_k)\notag\\
&(\omega_k|-)\sim \mbox{Dir}\left(\eta+ \sum_{j=1}^J\sum_{i=1}^{N_j}\delta(z_{ji}=k,v_{ji}=1),  \cdots,\eta+\sum_{j=1}^J\sum_{i=1}^{N_j}\delta(z_{ji}=k,v_{ji}=V)\right).
\end{align}

\section{Marked-Gamma-NB}
The Marked-Gamma-NB process model is constructed as
\begin{align}
x_{ji} &\sim F(\omega_{z_{ji}}), ~~\omega_k \sim \mbox{Dir}(\eta,\cdots,\eta)\notag\\
N_j &= \sum_{k=1}^K n_{jk},~~n_{jk}\sim \mbox{Pois}(\lambda_{jk}),~\lambda_{jk}\sim \mbox{Gamma}(r_k,p_k/(1-p_k))\notag\\
r_k&\sim\mbox{Gamma}(\gamma_0/K,1/c),~~p_k\sim\mbox{Beta}(a_0,b_0),~~\gamma_0\sim \mbox{Gamma}(e_0,1/f_0).
\end{align}
The block Gibbs sampling can be expressed as
\begin{align}
&\mbox{Pr}(z_{ji}=k|-) \propto F(x_{ji};\omega_{k})\lambda_{jk}\notag\\
&p_k\sim \mbox{Beta}\left(a_0+\sum_{j=1}^Jn_{jk},b_0+ Jr_k\right),~~p'_k=\frac{-J\ln(1-p_k)}{c-J\ln(1-p_k)}\notag\\
&l_{jk} \sim\mbox{CRT}(n_{jk},r_k),~l'_k \sim\mbox{CRT}(\sum_{j=1}^Jl_{jk},\gamma_0/K),~\gamma_0\sim\mbox{Gamma}\left(e_0+\sum_{k=1}^Kl'_k,\frac{1}{f_0-\sum_{k=1}^K\ln(1-p'_k)/K}\right)\notag\\
&(r_k|-) \sim \mbox{Gamma} \left( \gamma_0/K+  \sum_{j=1}^J  l_{jk}, \frac{1}{c - J \ln( 1 - p_k)}\right),~~(\lambda_{jk}|-) \sim \mbox{Gamma}(r_k+n_{jk},p_k)\notag\\
&(\omega_k|-)\sim \mbox{Dir}\left(\eta+ \sum_{j=1}^J\sum_{i=1}^{N_j}\delta(z_{ji}=k,v_{ji}=1),  \cdots,\eta+\sum_{j=1}^J\sum_{i=1}^{N_j}\delta(z_{ji}=k,v_{ji}=V)\right).
\end{align}

%

\end{document}